\documentclass{article}
\DeclareUnicodeCharacter{0302}{\^}
\usepackage[utf8]{inputenc}
\usepackage{xcolor}
\usepackage{amsmath}
\usepackage{amsthm}
\usepackage{amsfonts}
\usepackage{amssymb}
\usepackage{algorithm}
\usepackage{algorithmic}
\usepackage{graphicx}
\usepackage{float}  
\usepackage{placeins}
\usepackage[ left=3.00cm, right=3.00cm]{geometry}
\usepackage[style=numeric, maxnames=99, backend=biber]{biblatex}
\usepackage{hyperref}
\hypersetup{
    colorlinks=true,
    linkcolor=blue,
    urlcolor=blue,
    citecolor=blue
}
\newtheorem{mydef}{Definition}
\newtheorem{thm}{Theorem}
\newtheorem{claim}{Claim}
\newtheorem{rmk}{Remark}

\newtheorem{lemma}{Lemma}
\newtheorem{prop}{Proposition}
\begin{document}
\title{ Procrustes Wasserstein Metric: A Modified Benamou-Brenier Approach  with Applications to Latent Gaussian Distributions }

\author{Kevine Meugang Toukam\footnote{Universität Leipzig, Fakultät für Mathematik und Informatik, Augustusplatz 10, 04109 Leipzig,
Germany and Max Planck Institute for Mathematics in the Sciences, 04103 Leipzig, Germany}}
    
\date{\today}
\maketitle
\begin{abstract}
In this work, we introduce a modified Benamou-Brenier type approach leading to a Wasserstein type distance that allows global invariance, specifically, isometries, and we show that the problem can be summarized to orthogonal transformations. This distance is defined by penalizing the action with a costless movement of the particle that does not change the direction and speed of its trajectory. We show that for Gaussian distribution  resume to measuring the Euclidean distance between their ordered vector of eigenvalues and we show a direct application in recovering Latent Gaussian distributions.
\end{abstract}


\section{Introduction}

The study of data similarity is a cornerstone of modern data science, with applications ranging from network analysis and computer vision to bioinformatics \cite{yan2016short}. A meaningful notion of similarity must not only capture structural and functional relationships but also account for inherent symmetries and transformations in the data. Optimal Transport (OT) has emerged as a powerful framework for comparing data distributions, offering a principled way to align structures while preserving essential geometric and statistical properties. Recent advances, such as the Gromov-Wasserstein (GW) framework \cite{memoli2011gromov}, have extended OT to settings where data lie in incomparable spaces, enabling comparisons based on intrinsic relational structures and providing a natural way to handle invariance to isometries. While GW focuses on relational invariance, our work addresses invariance to explicit geometric transformations such as rotations and translations in Euclidean spaces, offering a complementary perspective.

A critical challenge in OT-based data comparison is the notion of invariance. Real-world data often exhibit symmetries such as permutations, rotations, or translations that should not affect their similarity measure. Existing approaches \cite{fiori2015spectral, petric2019got} have addressed invariance by incorporating permutation constraints, where alignments are represented by permutation matrices. While elegant, these methods often rely on stochastic optimization algorithms, such as stochastic gradient descent, to approximate solutions. Although effective in practice, such approaches can converge to suboptimal solutions due to the complexity of the optimization landscape.

In this article, we adopt a novel approach by focusing on the \textit{action of the Euclidean group} to illustrate geometric invariance in $\mathbb{R}^d$. This perspective generalizes the notion of invariance beyond permutations, incorporating rotations and translations to capture a broader class of symmetries. Although the work of \cite{alvarez2019towards} nicely incorporates those invariance in the Wasserstein metric and even provides a nice algorithm to numerically address this problem, there is no information about a possible explicit solution for a more tractable situation, namely the Gaussian case. Using this group action, we simplify the formulation of the OT problem while maintaining key invariance properties, offering a path to exact solutions. Our contributions are summarized as follows:

\begin{itemize}
    \item[-] \textbf{Dynamic Formulation}: We propose a novel formulation of OT under the action of the Euclidean group in $\mathbb{R}^d$ from a fluid dynamics perspective, inspired by the Benamou-Brenier framework.
    \item [-] \textbf{Static Formulation}: We derive a static formulation of the problem, analogous to the Kantorovich formulation in classical OT, which we call the \textit{Procrustes-Wasserstein distance}.
    \item [-] \textbf{Explicit Solution for Gaussian Distributions}: We provide an explicit expression for the Procrustes-Wasserstein distance between Gaussian distributions, offering key theoretical insights.
    \item [-] \textbf{Application to data recovery}: We address the problem of recovering latent Gaussian distributions from observed data transformed by an unknown orthogonal matrix. Using the Procrustes-Wasserstein distance, we derive an estimator for the equivalence class of the latent distribution, accounting for orthogonal invariance. Our method efficiently estimates the covariance structure in an asymptotically unbiased manner, enabling recovery of the latent distribution up to an orthogonal transformation as the sample size grows.
\end{itemize}
This work bridges theoretical advancements in OT with practical applications in data recovery, demonstrating the power and versatility of the Procrustes-Wasserstein distance in capturing meaningful data similarities.
\section{Reminder: Wasserstein distances, duality and Benamou-Brenier distance}\label{section2}
\subsection{The Monge formulation of the Optimal Transport problem}
Monge's formulation of the optimal transport problem \cite{monge1781memoire} is intuitive and straightforward. Given two probability measures, $\mu_0$ and $\mu_1$, on two distinct metric spaces, $X$ and $Y$ respectively, the goal is to find a transport map $T: X \rightarrow Y$ that pushes the measure $\mu_0$ forward to the measure $\mu_1$, minimizing the total transportation cost.

\textbf{Mathematical Formulation}:
Let $X$ and $Y$ be two Polish spaces, and $\mu_0$ and $\mu_1$ be probability measures defined on them. The optimal transport problem can be mathematically formulated as finding a measurable map $T: X \rightarrow Y$ such that $T_\#\mu_0 = \mu_1$, and minimizing the cost function $C(T)$ given by:

\[ C(T) = \int_X c(x,T(x)) d\mu_0(x),\]
where $c(x, y)$ is the cost of transporting a unit of mass from location $x$ in $X$ to location $y$ in $Y$.

    While Monge's formulation is elegant and conceptually clear, it is challenging to find optimal solutions, especially when dealing with complex measures and high-dimensional spaces. One significant limitation of Monge's formulation is that it requires the existence of a well-defined transport map $T(x)$ for every $x \in X$. However, in many practical scenarios, such a transport map may not exist. 
\subsection{Original Kantorovich's problem}
Let us start with the mathematical formulation. Given two Polish spaces $X$ and $Y$, and a cost function $c : X \times Y \rightarrow [0, +\infty]$. For simplicity, we assume that $c$ is continuous and symmetric. The goal is to find a solution to the following problem:

\[
\min_{\gamma \in \Pi(\mu_0, \mu_1)} \int_{X \times Y} c\, d\gamma \label{KP}\tag{KP}
\]
where $\mu_0$ and $\mu_1$ are probability measures on $X$, and $\Pi(\mu_0, \mu_1)$ is the set of transport plans, defined as:

\[
\Pi(\mu_0, \mu_1) = \{\gamma \in P(X \times Y) : (\pi_0)_\#\gamma = \mu_0, (\pi_1)_\#\gamma = \mu_1\} \tag{2}
\]

Here, $(\pi_0)_\#\gamma$ and $(\pi_1)_\#\gamma$ denote the marginal distributions of $\gamma$ with respect to the first and second components of $X \times Y$, respectively.

In the traditional Monge formulation, one seeks a transport map $T : X \rightarrow Y$ that assigns a unique destination $T(x)$ for each point $x$. However, in Kantorovich's formulation\cite{kantorovich1942translocation}, the focus is on transport plans $\gamma$, which describe the movement of particles from one point to possibly multiple targets, allowing for more general and flexible transport patterns.

The solutions to \ref{KP} are referred to as optimal transport plans between $\mu_0$ and $\mu_1$.
\begin{rmk}
    If a transport plan $\gamma$ can be represented as $(\text{id}, T)_\#\mu_0$ for a measurable map $T : X \rightarrow X$, then $T$ is called an optimal transport map from $\mu_0$ to $\mu_1$. This plus the fact that $ \int_{X \times Y} c(x,y)\, d\gamma = \int_{X} c(x,T(x))\, d\mu_0 $ allow us to see in some sense the Kantorovich formulation as a generalization of the Monge's problem.
\end{rmk}
Then, the generalization introduced by Kantorovich makes the problem more tractable compared to the original Monge formulation. Instead of searching for a unique transport map, the focus is now on finding optimal transport plans, which always exist for Kantorovich's problem.
\subsection{Dual Problem}

The dual formulation of the Kantorovich problem involves finding dual variables $\phi$ and $\psi$ associated with the measures $\mu_0$ and $\mu_1$, respectively. The dual problem (DP) is given by:
\[
\text{(DP)} \quad \sup \left\{ \int \phi(x) \, d\mu_0 + \int \psi(y) \, d\mu_1 : \phi(x) + \psi(y) \leq c(x, y) \text{ for all } (x, y) \in X \times X \right\}
\]

Here, $\phi$ and $\psi$ are real-valued functions on $X$. Any pair of maximizers (if they exist) are called \textit{Kantorovich potentials}.
These functions play a crucial role in the optimal transport problem and are related to the transport cost and the existence of optimal transport plans.

The duality result between (KP) and (DP) is given by the Kantorovich-Rubinstein duality theorem \cite{santambrogio2015optimal}:

\[
\inf \text{(KP)} = \sup \text{(DP)}.
\]
\subsection{Wasserstein space}
Starting from the problem (KP) described above with $X=Y$ the given Polish space, we can define a new set of distances within the realm of probability measures, denoted as $P(X)$. Our primary focus revolves around costs represented by the function $c(x, y) = |x - y|^p$, designed for the metric space $ \Omega \subset X = \mathbb{R}^d$, with 
$\Omega$ convex. This approach can be extended to more general metric spaces by incorporating the concept of distance raised to the power of $p$, where $p$ falls within the range of $[1, +\infty)$.

In scenarios where $\Omega$ is unbounded, the analysis is constrained to a specific probability set defined as follows:
\[ P_p(\Omega) := \{ \mu \in P(\Omega): \int |x|^p \, d\mu(x) < +\infty \}. \]

When working on a general metric space $X$, an arbitrary point $x_0 \in X$ is chosen as a reference, leading to the following definition:
\[ P_p(X) := \{ \mu \in P(X): \int d(x, x_0)^p \, d\mu(x) < +\infty \}. \]
Importantly, the finiteness of this integral (and therefore the definition of $P_p(X)$) is independent of the  chosen $x_0$.

Our primary focus centers on distances defined by the formula for $p=2$:
\begin{equation}\label{continuoustransportmetric}
     (W_p(\mu_0, \mu_1))^p = \inf \{ \int_{X \times X} |x - y|^p   d\pi(x,y) ~ , {\pi \in \Gamma(\mu_0, \mu_1)} \}.
\end{equation}
These functions are referred to as Wasserstein distances for every $p$ \cite{villani2009optimal}.
\subsection{Wasserstein between Gaussian distributions}
Finding optimal transport plans between probability distributions is often challenging. However, in certain scenarios, explicit solutions are available. For instance, in one dimension (where $n = 1$), if the cost function $c$ is convex and based on the Euclidean distance along the line, the optimal plan involves a monotonic rearrangement of the distribution $\mu_0$ into $\mu_1$. This means that mass is transported in a monotonic manner from left to right. (Refer to Chapter 2, Section 2.2 of \cite{villani2021topics} for detailed explanations.) Another situation where a solution is known, particularly for a quadratic cost, is in the Gaussian case, applicable in any dimension $n \geq 1$. Here, the cost function is quadratic, and an explicit solution exists \cite{dowson1982frechet, takatsu2010wasserstein}.

Given $\mu_i = \mathcal{N}(m_i, \Sigma_i)$, $i \in \{0, 1\}$ two Gaussian distributions on $\mathbb{R}^d$, the 2-Wasserstein distance $W_2$ between $\mu_0$ and $\mu_1$ has a closed-form expression, well known as the Burg or Frechet metric ( section 1 of \cite{dowson1982frechet})  which can be written as:
\begin{equation}\label{burgmetric}
W_2^2(\mu_0, \mu_1) = \|m_0 - m_1\|^2 + \mathrm{tr}\left(\Sigma_0 + \Sigma_1 - 2\left(\Sigma_1^{1/2}\Sigma_0\Sigma_1^{1/2}\right)^{1/2}\right),
\end{equation}
where, for every symmetric semi-definite positive matrix $M$, the matrix $M^{1/2}$ is its unique semi-definite positive square root.
\subsection{Benamou-Brenier Formulation}

Benamou and Brenier \cite{benamou2000computational} introduced an alternative numerical framework for the optimal mass transfer problem, linking (KP)  to continuum mechanics. They study the dynamical problem from measure $\mu_0$ at $t = 0$ to $\mu_1$ at $t = 1$. In the setting $X = Y = \mathbb{R}^d$ with squared Euclidean cost $c(x, y) = \|x - y\|^2$, solving the problem coincides with finding the minimal path $(\mu_t)_{t=0}^1$, minimizing total length.

The path $\mu_t$ is described through a time-varying vector field $v(t, \cdot)$ satisfying the continuity equation $\frac{\partial \mu_t}{\partial t} + \nabla \cdot (\mu_t V) = 0$, $\mu_0 = \mu_0$, $\mu_1 = \mu_1$. This vector field $V(t, \cdot)$ represents the speed, and $\mu_t v(t, \cdot)$ corresponds to momentum.

Reformulating the optimal transportation problem in a differential way, inspired by fluid mechanics, is crucial for studying dynamical problems. Each curve $\mu_t$ represents the measure's evolution over time, interpreted as fluid flow along a family of vector fields.

We search for the vector field $V(t, \cdot)$ satisfying  conservation of mass and  minimizing kinetic energy. The infinitesimal length of such a vector field can be computed as \\ $\|V\|^2_{L^2(\mu_t)} = \left(\int_{\mathbb{R}^d} \|V(t, x)\|^2 d\mu_t(x)\right)^{1/2}$. 

This results in the minimal-path reformulation of the problem:
\begin{equation}\label{BB metric}
(W_2(\mu_0, \mu_1))^2 = \inf_{(\mu_t, V)} \int_0^1 \int_{\mathbb{R}^d} \frac{1}{2} \|V(t, x)\|^2 d\mu_t(x) dt
\end{equation}
subject to $\frac{\partial \mu_t}{\partial t} + \nabla \cdot (V\mu_t) = 0$, $\mu_0 = \mu_0$, $\mu_1 = \mu_1$.

The path $\mu_t$ describes the time-evolving density of particles moving continuously with velocity $V(t, \cdot)$.

\section{A Modified Benamou-Brenier Distance}\label{section3}
From a fluid dynamics perspective, this kinetic energy represents the effort needed to move particles according to the vector field $V$. The Benamou-Brenier formulation selects the vector field $V$ minimizing the total efforts or kinetic energy required for particle movement. By looking at the minimal energy required for all geometric structures preserved during the particle movement according to the vector field $V$, we propose a novel formulation of the Benamou-Brenier distance that allows for global geometric invariance to be incorporated into the objective function.

\textbf{Mathematical Formulation}:
\noindent Given $\mu_0$ and $\mu_1$ be probability measures defined on  $\mathbb{R}^d$. The Modified Benamou-Brenier problem (MBB) can be mathematically formulated as :
\begin{equation}\label{continuousMBB}
   \bar{d}^2(\mu_0,\mu_1)= \inf \left\{ \int_{0}^{1} \int_{\mathbb{R}^{d}}\frac{1}{2} \| V_{t} + L_{E_t \theta_t } - C_{t} \|^{2} \, \mu_{t} \, dx \, dt  \ | (u_t,V_t,C_t, E_t, \theta_t) \in A \right\}
\end{equation}
where, $A= \{(u,V,C, E, \theta) \ | \ \forall t \in [0,1], \  \partial_t \mu_t + \text{div} (u_t \cdot V_t) = 0, \ C_t \in \mathbb{R}^d, \ E_t \in \text{Skew}(d), \ \dot{\theta}_t = E_t \theta_t, \ 
\theta_{t=0} = I_d,  \ \mu_{t=0,1}=\mu_{0,1}, \ t\in [0,1] \}$, where, 
\begin{eqnarray*}
    \text{Skew}(d)=  \{ \Omega \in \mathbb{R}^{d \times d} \ | \ \Omega^\top = -\Omega\} ,
\end{eqnarray*}
  
and $\theta_t $ satisfies the evolution equation:
\begin{eqnarray}\label{evolutionequationoforthogonalcurve}
\begin{cases}
\dot{\theta}_t = E_t \theta_t, \\
\theta_0 = I_d,
\end{cases}
\end{eqnarray}
where \( I_d \) is the \( d \times d \) identity matrix.

\begin{prop}
    Given $\mu_0$,  $\mu_1$ $\in  P_2(\mathbb{R}^d)$, the Modified Benamou-Brenier problem \eqref{continuousMBB} admits a static formulation Wasserstein like given as follow;
    \begin{equation} \label{staticMBB}
        \bar{d}^2(\mu_0,\mu_1)= \inf_{\theta \in O(d)} \left\{ ( W_2(\theta_{\#} \bar{\mu_0}, \bar{\mu_1}))^2\right\}.
    \end{equation}
    where $\bar{\mu_{i}}, \ i\in \{0,1\} $ represent the centered measure $\mu_i, \ \forall i \in \{0,1\}$. This will be called the Procrustes Wasserstein metric.
\end{prop}
\begin{proof}
    The proof requires knowledge from optimal control, duality of the Hamilton Jacobi equation,and duality of the Kantorovich formulation of Optimal transport. We start by writing the Lagrangian of the problem, we have:
    \begin{equation}\label{lagrangian}
        L(\mu, V, C, E,  \theta, \lambda)= \int_{0}^{1} \int_{\mathbb{R}^{d}} \frac{1}{2} \| V_{t} + L_{E_t  \theta_t} - C_{t} \|^{2} \mu_t - \lambda(\partial_t \mu_t + \text{div} (u_t \cdot V_t)) \, \, dx \, dt .
    \end{equation}
   By integration by parts, \eqref{lagrangian} becomes,
   \begin{align*}
     L(\mu, V, C, E,  \theta, \lambda) &=  \int_{0}^{1} \int_{\mathbb{R}^{d}} (\frac{1}{2} \| V_{t} + L_{E_t  \theta_t} - C_{t} \|^{2}- \partial_t \lambda - \nabla \lambda \cdot V_t) \, \mu_{t} \, dx \, dt\\
     &+  \int_{\mathbb{R}^{d}} \lambda(1,x)\mu_1(x) - \lambda(0,x) \mu_0(x) \, dx,\\
     &= \int_{\mathbb{R}^{d}} (\frac{1}{2} \| V_{t} + L_{E_t  \theta_t} - C_{t} \|^{2}- \partial_t \lambda - \nabla \lambda \cdot V_t) \, \mu_{t} \, dx \, dt + \underbrace{\langle \lambda_1, \mu_1 \rangle - \langle \lambda_0, \mu_0 \rangle}_{B},\\
     &= \int_{\mathbb{R}^{d}} (\frac{1}{2} \| V_{t} + L_{E_t  \theta_t} - C_{t}- \nabla \lambda \|^{2}- \partial_t \lambda -\frac{1}{2}\|\nabla \lambda\|^2- \nabla \lambda \cdot (-L_{E_t  \theta_t}+ C_t)) \, \mu_{t} \, dx \, dt + B.
   \end{align*}
  Using the Lagrangian method, problem \eqref{continuousMBB}  is equivalent to 
  \begin{equation} \label{minlagrangian}
     \inf_{u,V,C, E, \theta}  \sup_{\lambda} L(\mu, V, C, E, \theta,  \lambda)
  \end{equation}
   It follows from the fact that $L$ is quadratic in $V$ that the optimal value of $V$ namely $V^*=  C_{t}- L_{E_t  \theta_t} + \nabla \lambda  $ and we have
  \begin{equation*}
       L(\mu, V^*, C, E,  \theta, \lambda) = \int_{\mathbb{R}^{d}} (-\partial_t \lambda -\frac{1}{2}\|\nabla \lambda\|^2- \nabla \lambda \cdot (-L_{E_t  \theta_t}+ C_t)) \, \mu_{t} \, dx \, dt + B.
  \end{equation*}
  By duality of the MinMax principle, we have 
  \begin{equation*}
      \eqref{minlagrangian} \iff \inf_{\theta_t} \sup_\lambda  \inf_{C_t,E_t} \left\{  \langle \lambda_1, \mu_1 \rangle - \langle \lambda_0, \mu_0 \rangle \ | \ {\partial_t \lambda + \frac{1}{2}\|\nabla \lambda\|^2+ \nabla \lambda \cdot (-L_{E_t  \theta_t}+ C_t)= 0}  \right\}.
  \end{equation*}
  In order to write explicitly \eqref{minlagrangian}, we have to try to solve this time depending Hamilton Jacobi equation or simply, computes the value of the Lagrange multiplier $\lambda_1$ at time $t=1$.

  Given a general time dependent  HJB equation, that is 
  \begin{equation}\label{timedependingHJB}
      \begin{cases}
      \partial_t \lambda + H (t,x,\nabla \lambda),\\
      \lambda(0,\cdot)= \lambda_0,
      \end{cases}
  \end{equation}
  in our case, $H (t,x,\nabla \lambda)= \frac{1}{2}\|\nabla \lambda\|^2+ \nabla \lambda \cdot (-L_{E_t  \theta_t}+ C_t) $, it follows from \cite{nguyen2014layered} (Thm 3.1) that the solution of \eqref{timedependingHJB} is given by
  \begin{equation*}
      \lambda(t,x)= \sup_{q\in \mathbb{R}^d} \left \{ x\cdot q- \lambda_0^*(q) - \int_{0}^{t} H(s,q) ds \right\},
  \end{equation*}
  where $\lambda_0^* $ is the Fenchel conjugate of the initial condition of \eqref{timedependingHJB} given by $\lambda_0^* = \sup_{p\in \mathbb{R}^d}  \{p\cdot q - \lambda_0(p) \}$. According to \cite{nguyen2014layered} these solutions are called Layered viscosity solutions.

  Let us compute $\lambda_1(x)$ using the expression of our time depending HJB equation $H (s ,q)= \frac{1}{2}\|q\|^2+ q \cdot (-L_{E_s  \theta_s}+ C_s) $. 

  For every $x\in \mathbb{R}^n$, we have 
  \begin{align*}
      \lambda(1,x) &= \sup_{q\in \mathbb{R}^d} \left \{ x\cdot q- \lambda_0^*(q) - \int_{0}^{1} \frac{1}{2}\|q\|^2+ q \cdot (-L_{E_s  \theta_s}(x)+ C_s) ds \right\},\\
      &=\sup_{q\in \mathbb{R}^d} \left \{ x\cdot q- \lambda_0^*(q) -  \frac{1}{2}\|q\|^2- q \cdot (-L_{\bar{E  \theta} }(x)+ C )  \right\} \ \text{with} \ \int_{0}^{1} C_s ds = C, \ \int_{0}^{1} L_{E_s  \theta_s } ds = L_{\bar{E  \theta}  } ,\\
      &=\sup_{q\in \mathbb{R}^d} \left \{ x\cdot q +\inf_{p\in \mathbb{R}^n}  \{ \lambda_0(p) -p\cdot q \} -  \frac{1}{2}\|q\|^2- q \cdot (-L_{\bar{E  \theta}  }(x)+ C )  \right\} \\
      &= \inf_{p\in \mathbb{R}^d} \sup_{q\in \mathbb{R}^d} \left \{  \lambda_0(p) -  \frac{1}{2}\|q\|^2+ q \cdot (x+L_{\bar{E  \theta} }(x)- C  -p )  \right\} ,\\
      &= \inf_{p\in \mathbb{R}^d} \sup_{q\in \mathbb{R}^d} \left \{  \lambda_0(p) -  \frac{1}{2}(q-(x+ {\bar{E  \theta}  }x- C  -p ))^2+ \frac{1}{2} (x- {E }x- C  -p )^2  \right\},\\
      &= \inf_{p\in \mathbb{R}^d}  \left \{  \lambda_0(p) + \frac{1}{2} (x+ {\bar{E   \theta}  }x- C  -p )^2  \right\}.
  \end{align*}
  It is worth to notice that $\bar{E \theta}= \theta(1)-\theta(0)$, where $\theta $ is the solution of the ODE \eqref{evolutionequationoforthogonalcurve} well known as the ordered exponential map define by $\theta(t)=  \mathcal{T} \exp\left( \int_0^t E_s \, ds \right)$,
where $\mathcal{T}$ (the time-ordering operator) ensures that operators are multiplied in the correct chronological order. 
\begin{lemma}\label{orderedexponential map}
Given $\theta_t$ the ordered exponential map solution of the ODE \eqref{evolutionequationoforthogonalcurve}, 
     $\forall t\in [0,1]$, $\theta_t \in O(d)$ if and only if $E_t$ is a skew-symmetric matrix $\forall \ t $.
\end{lemma}
\begin{proof}

\begin{itemize}
    \item Let's assume  \( E_t^T = -E_t \), and show that \( \theta_t \) is orthogonal i.e., \( \theta_t^T \theta_t = I \).

The ordered exponential map is defined as:
\[
\theta_t = \mathcal{T} \exp\left( \int_0^t E_s \, ds \right),
\]
which satisfies the matrix differential equation:
\[
\frac{d\theta_t}{dt} = E_t \theta_t, \quad \theta_0 = I.
\]
 Differentiating \( \theta_t^T \theta_t \) with respect to \( t \), we get:
\[
\frac{d}{dt} \left( \theta_t^T \theta_t \right) = \frac{d\theta_t^T}{dt} \theta_t + \theta_t^T \frac{d\theta_t}{dt}.
\]
Substituting \( \frac{d\theta_t}{dt} = E_t \theta_t \) and his transpose \( \frac{d\theta_t^T}{dt} = \theta_t^T E_t^T \) into the derivative of \( \theta_t^T \theta_t \), we obtain:
\[
\frac{d}{dt} \left( \theta_t^T \theta_t \right) = \theta_t^T E_t^T \theta_t + \theta_t^T E_t \theta_t.
\]

 Using the skew-symmetry of \( E_t \), $E_t^T + E_t = 0$, leads us to 
\[
\theta_t^T E_t^T \theta_t + \theta_t^T E_t \theta_t = \theta_t^T \left( E_t^T + E_t \right) \theta_t =  0.
\]
Therefore:
\[
\frac{d}{dt} \left( \theta_t^T \theta_t \right) = 0.
\]
This implies \( \theta_t^T \theta_t \) is constant in time. Since \( \theta_0 = I \), we have:
\[
\theta_t^T \theta_t = I \quad \text{for all } t.
\]

Thus, \( \theta_t \) is orthogonal.

\item Let us suppose that \( \theta_t \) is orthogonal for all \( t \) and show that \( E_t^T = -E_t \).

The orthogonality condition of \( \theta_t \) for all \( t \) means;
\[
\theta_t^T \theta_t = I \quad \text{for all } t.
\]
By differentiating both sides with respect to \( t \), we get:
\[
\frac{d}{dt} \left( \theta_t^T \theta_t \right) = \frac{d\theta_t^T}{dt} \theta_t + \theta_t^T \frac{d\theta_t}{dt} = 0.
\]

As before,  Substituting \( \frac{d\theta_t}{dt} = E_t \theta_t \) and his transpose
 into the derivative of \( \theta_t^T \theta_t \) and consecutively factorizing, we have:
\[
\theta_t^T E_t^T \theta_t + \theta_t^T E_t \theta_t = \theta_t^T \left( E_t^T + E_t \right) \theta_t = 0.
\]
Since \(  \theta_t \) is invertible as an orthogonal matrix, we can factor it out to:
\[
E_t^T + E_t = 0.
\]
This implies that \( E_t \) is skew-symmetric.
\end{itemize}
\end{proof}
  Using Lemma \ref{orderedexponential map}, and replacing the value of the ordered exponential map at 0 and 1 leads us to the following computations.
  \begin{align*}
      \inf_{\theta} \sup_{\lambda} \inf_{u,C} L(\mu, V, C, \theta, \lambda)&=  \inf_{C,\theta} \sup_{\lambda_0, \lambda_1} \left\{  \langle \lambda_1, \mu_1 \rangle - \langle \lambda_0, \mu_0 \rangle \ | \  \lambda_1(x)-\lambda_0(y)\leq \frac{1}{2} ({\theta }x- C  -y )^2    \right\},\\
       &= \inf_{C,\theta} \inf_{\Pi \in \Gamma(\mu_0,\mu_1)} \left\{ \int_{\mathbb{R}^d\times\mathbb{R}^d} \frac{1}{2} ( {\theta }x -y - C)^2 d\Pi(x,y)   \right\},\\
       &=\inf_{\theta \in O(d)} \inf_{\Pi \in \Gamma(\mu_0,\mu_1)}  { \inf_{C} \left\{ \int_{\mathbb{R}^d\times\mathbb{R}^d} \frac{1}{2} ( {\theta }x -y - C)^2 d\Pi(x,y)   \right\}}.
  \end{align*}
  The following claim provides the optimal value of C.
  \begin{claim}
      Given $\mu_0, \ \mu_1 $as above, for every $\theta \in O(d)$ and $\Pi \in \Gamma(\mu_0,\mu_1)$, the optimal value $C^*$ that minimizes $PbC$
    \begin{align*}
    PbC =\inf_{{C}}\int_{\mathbb{R}^d\times\mathbb{R}^d} \frac{1}{2}|\theta x-y-{C}|^2d\Pi(x,y),
\end{align*}
      is given by 
      \begin{equation*}
          C^*= \langle x\rangle_{\theta_\# \mu_0 }-\langle y\rangle_{\mu_1}.
      \end{equation*}
  \end{claim}

Let us define the function  $g:\mathbb{R}^d\times\mathbb{R}^d\to\mathbb{R}^d$ as $g(x,y):=\theta x-y$, then
\[
PbC=\inf_{{C}}\int_{\mathbb{R}^d\times\mathbb{R}^d}\frac{1}{2}|g(x,y)-{C}|^2d\Pi(x,y).
\]
From probability analysis, precisely the definition of variance, the optimal value of ${C}$ in the vector expectation of $g$ given the joint probability distribution $\Pi$.
\begin{align*}
{C^*}&=\langle g(\cdot,\cdot)\rangle_{\Pi},\\
&=\langle  \theta x-y\rangle_{\Pi},\\
&=\langle  \theta x\rangle_{\mu_0}-\langle y\rangle_{\mu_1}.
\end{align*}

Thus, our minimization problem \eqref{minlagrangian} becomes
\begin{align*}
     \sup_{\lambda} \inf_{u,V,C, \theta} L(\mu, V, C, \theta, \lambda)&=\inf_{\theta} \inf_{\Pi \in \Gamma(\mu_0,\mu_1)} \int_{\mathbb{R}^d\times\mathbb{R}^d} \frac{1}{2}|\theta x-y-{C^*}|^2d\Pi(x,y),\\
    &=\inf_{\theta} \inf_{\Pi \in \Gamma(\mu_0,\mu_1)} \int_{\mathbb{R}^d\times\mathbb{R}^d} \frac{1}{2}|\theta x-\langle x\rangle_{\theta_\# {\mu_0}}-(y-\langle y\rangle_{\mu_1})|^2d\Pi(x,y),\\
    &=\inf_{\theta} \inf_{\bar{\Pi} \in \Gamma(\mu_0,\mu_1)} \int_{\mathbb{R}^d\times\mathbb{R}^d} \frac{1}{2}|z_1-z_2|^2d\bar{\Pi}(z_1,z_2), \\
    &=\inf_{\theta} \inf_{{\Pi} \in \Gamma( \theta_\#{\mu_0},\mu_1)} \int_{\mathbb{R}^d\times\mathbb{R}^d} \frac{1}{2}|z_1-z_2|^2d\Pi\\
    &=\inf_{\theta}  W^2_2(\theta_\# \bar{\mu_0},\bar{\mu}_1),\\
    &= \bar{d}^2(\mu_0,\mu_1)= \inf_{\theta} \left\{ ( W_2(\theta_\# \bar{\mu_0}, \bar{\mu_1}))^2\right\}.
\end{align*}
\end{proof}

\section{Study of the general problem}
Hence, if we want to solve the problem \eqref{continuousMBB}, we need to find the optimal isometry $\theta$ so that in order to minimize the cost from moving from $\mu_0$ to $\mu_1$, one should first look at the optimal geometrical invariant of $\mu_0$ that minimize the energy.
To do this, we need to better understand the structure of our underlined quotient space. We therefore begin by presenting the preserved structures inherited from Wasserstein space in a direct way, and we study those that are not necessarily preserved, such as geodesic in the Procruste Wasserstein setting.
\\

 Define an equivalence relation \(\sim\) on \(\mathcal{P}_2(\mathbb{R}^d)\) by setting \(\mu \sim \nu\) if and only if there exists an orthogonal transformation \(O \in \text{Isom}(\mathbb{R}^d)\) such that \(\nu = O_\# \mu\), where \(O_\# \mu\) denotes the pushforward measure. Let \(\pi: \mathcal{P}_2(\mathbb{R}^d) \to \mathcal{P}_2(\mathbb{R}^d) / \sim\) be the canonical projection.

Define the pseudometric \(\bar{d}(\mu, \nu) = \inf_{O \in \text{Isom}(\mathbb{R}^d)} W_2(\mu, O_\# \nu)\). The quotient space \(\mathcal{P}_2(\mathbb{R}^d) / \sim\) is equipped with the induced metric \({d}'\), defined by
\begin{eqnarray}
 d'([\mu],[\nu]) &=& \bar{d}(\mu, \nu) ,  \label{quotientmetric}\\
 &=& \inf_{O \in \text{Isom}(\mathbb{R}^d)} W_2(\mu, O_\# \nu).\notag
\end{eqnarray}
where \([\mu]\) denotes the equivalence class of \(\mu\).
\begin{prop}
    The space \( \mathcal{P}_2(\mathbb{R}^d)  \) endowed with the \( \bar{d} \) defined in \eqref{staticMBB} is a pseudo-metric space.
\end{prop}
\begin{proof}
  It follows from the fact that the transposition function is a bijection in the set of orthogonal matrices and the metric properties of the Wasserstein distance.
\end{proof}

\subsection{Topological properties}
The following theorem is a direct consequence of properties of the Wasserstein space.
\begin{thm}
    The space $\left( \mathcal{P}_2(\mathbb{R}^d)/ \text{Isom}(\mathbb{R}^d), \bar{d} \right)$ is complete and connected metric space.
\end{thm}

Let us study and characterize geodesics in the procruste Wasserstein setting. \\
The following result refines existing ideas about Wasserstein distances and quotient spaces by focusing on minimizing curves \cite{kloeckner2010geometric, burago2001course, do1992riemannian, gallot2004differential}.
\begin{thm}\label{Bureprocrustegeodesic}
Let \(\mathcal{P}_2(\mathbb{R}^d)\) be the space of probability measures on \(\mathbb{R}^d\) with finite second moments, equipped with the Wasserstein-2 distance \(W_2\). Let \(\text{Isom}(\mathbb{R}^d)\) be the group  isometry  \(\mathbb{R}^d\), and consider the quotient space \(\mathcal{P}_2(\mathbb{R}^d) / \sim\) with the metric $\bar{d}$ defined in \eqref{quotientmetric}

1. \textbf{Length Inequality}:  
   For any curve \(c: [0, 1] \to \mathcal{P}_2(\mathbb{R}^d)\),
   \[
   L(\pi \circ c) \leq L(c).
   \]

2. \textbf{Minimizing Curves for Procruste Wasserstein}:  
   Let \(c: [0, 1] \to \mathcal{P}_2(\mathbb{R}^d)\) be a minimizing curve in \(\mathcal{P}_2(\mathbb{R}^d)\) with respect to \(W_2\), and let \(\mu = c(0)\). Suppose there exists an optimal orthogonal transformation \(O^* \in \text{Isom}(\mathbb{R}^d)\) such that:
   \[
   \bar{d}([\mu], [\nu]) = W_2(\mu, O^*_\# \nu) \quad \text{and} \quad O^*_\# \nu = c(1).
   \]
   Then:
   \[
   L(\pi \circ c) = L(c) = W_2(\mu, \nu) = \bar{d}([\mu], [\nu]).
   \]
   Moreover, \(\pi \circ c\) is a minimizing curve in \(\mathcal{P}_2(\mathbb{R}^d) / \sim\).
\end{thm}

\subsection{Case of Gaussian distributions}
Given $\mu_i = \mathcal{N}(m_i, \Sigma_i)$, $i \in \{0, 1\}$ two Gaussian distributions on $\mathbb{R}^d$, \eqref{burgmetric} provides the 2-Wasserstein distance $W_2$ between $\mu_0$ and $\mu_1$. As it only required the expressions of the two covariance matrix, we only need one of $\Bar{\mu_0}\circ ( \theta)$ and $ \Bar{\mu_1}$. That is, respectively $\bar{\Sigma}_0= \theta\Sigma_0 \theta^T$ and $\Bar{\Sigma_1}= \Sigma_1$. 
And the problem reduce to 
\begin{eqnarray}\label{modifiedgaussian}
\bar{d}^2(\mu_0,\mu_1)&=& \inf_{\theta} \left\{  \mathrm{tr}\left(\Sigma_1 + \theta\Sigma_0 \theta^T - 2\left(\Sigma_1^{1/2} \theta\Sigma_0 \theta^T \Sigma_1^{1/2}\right)^{1/2}\right) \right\},\notag \\
&=&  \mathrm{tr}(\Sigma_1 ) + \mathrm{tr}(\Sigma_0 ) -2  \sup_{ \theta}\{ F(\theta)\},
\end{eqnarray}
With $F(\theta) = \mathrm{tr}\left( \left(\Sigma_1^{1/2} \theta\Sigma_0 \theta^T \Sigma_1^{1/2}\right)^{1/2}\right) $.
The following theorem provides an explicit and formula to compute the Procrustes Wasserstein distance between two Gaussian distributions.
\begin{thm}\label{gaussianprowass}
    Given $\mu_i = \mathcal{N}(m_i, \Sigma_i)$, $i \in \{0, 1\}$ two Gaussian distributions in $\mathbb{R}^d$, we represent their eigen decomposition by $\Sigma_i = P_i A_i {P_i}^T$ for $i \in \{0, 1\}$, with $A_i$ being a diagonal matrix with diagonal vector $a_i$, the Procrustes Wasserstein distance between $\mu_0$ and $\mu_1 $ is equal to the Euclidean distance between their ordered vector of eigen value. Namely given vector $a_i$ the ordered vector $ a_{i,1}\leq \cdot \cdot\cdot \leq a_{i,d}$ consists of eigen value of $\Sigma_i$  for $i \in \{0, 1\}$, we have 
    \begin{eqnarray*}
        \bar{d}^2(\mu_0,\mu_1)&=& \|\sqrt{a_0} - \sqrt{a_1}\|^2.
    \end{eqnarray*}
    The optimal orthogonal transformation is given by $P_0^T P_1$ and the optimal Monge map is the one between $\Sigma_1 $ and $ P_0^T P_1\Sigma_0 P_1^T P_0 $.
 \end{thm}
\begin{proof}
    Let us consider  $\mu_i = \mathcal{N}(m_i, \Sigma_i)$, $i \in \{0, 1\}$ two Gaussian distributions on $\mathbb{R}^d$, we represent their eigen decomposition by $\Sigma_i = P_i A_i {P_i}^T$ for $i \in \{0, 1\}$, with $A_i$ being a diagonal matrix with diagonal vector $a_i$ the ordered vector $ a_{i,1}\leq \cdot \cdot\cdot \leq a_{i,d}$ constitutes eigen value of $\Sigma_i$ and $P_i$ is an orthogonal matrix.
    \begin{lemma}\label{setofeigenvalue}
        Given $ M$, $N$ two square matrix Positive semidefinite. If $\{eiv(M)\} = \{eiv(N)\}$, then $\mathrm{tr}(M) = \mathrm{tr}(N)$ and $\mathrm{tr}(M^{\frac{1}{2}}) = \mathrm{tr}(N^{\frac{1}{2}})$.
    \end{lemma}
    If replace the eigen decomposition of $\Sigma_i$, $  \forall i \in \{0,1\} $, we obtain;
    \begin{eqnarray*}
        \Sigma_1^{1/2} \theta\Sigma_0 \theta^T \Sigma_1^{1/2} &=& P_1 \underbrace{ A_1^{1/2}P_1^T \theta P_0 A_0 P_0^T \theta^T P_1 A_1^{1/2} }_{T_\theta {T_\theta}^T} P_1^T,  \\
         &=& P_1 A_1^{1/2} \underbrace{P_1^T \theta P_0}_{\Theta} A_0 \underbrace{ P_0^T \theta^T P_1}_{{\Theta}^T} A_1^{1/2} P_1^T,\\
         &=& P_1 A_1^{1/2} {\Theta} A_0 {{\Theta}^T} A_1^{1/2} P_1^T,\\
          &=& P_1 \underbrace{A_1^{1/2}  \Theta A_0^{1/2} }_{T_\Theta} \underbrace{ A_0^{1/2} \theta^T  A_1^{1/2} }_{{T_\Theta}^T}P_1^T,\\
          &=& P_1 T_\Theta T_\Theta ^T P_1^T.
    \end{eqnarray*}
    This is simply equivalent to say that $\{eiv(\Sigma_1^{1/2} \theta\Sigma_0 \theta^T \Sigma_1^{1/2})\} = \{eiv(T_\Theta T_\Theta ^T)\}$ and thus $ \mathrm{tr} \left( (\Sigma_1^{1/2} \theta\Sigma_0 \theta^T \Sigma_1^{1/2})^\frac{1}{2} \right)\} = \{\mathrm{tr} \left( (T_\Theta T_\Theta ^T)^\frac{1}{2} \right) $ by application of Lemma \ref{setofeigenvalue} and the the fact that the eigen value does not depend of the basis where the matrix is represented.
    Let's recall that given the function  $F(\cdot)$ defined above, we have
    \begin{eqnarray*}
        F(\theta )&=&  \mathrm{tr}\left( \left(\Sigma_1^{1/2} \theta\Sigma_0 \theta^T \Sigma_1^{1/2}\right)^{1/2}\right),\\
        &=& \sum_i^d eiv_i ( \left(\Sigma_1^{1/2} \theta\Sigma_0 \theta^T \Sigma_1^{1/2}\right)^{1/2}),\\
        &=&  \sum_i^d eiv_i ( \left(T_\Theta T_\Theta ^T\right)^{1/2}),\\
        &=&  \sum_i^d  \left(eiv_i (T_\Theta T_\Theta ^T) \right)^{1/2},\\
        &=& \sum_i^d  \left( \sigma_i(T_\Theta T_\Theta ^T) \right)^{1/2}, \text{ where } \sigma(M) \text{ denotes the sigular value of } M.
    \end{eqnarray*}
    Since $P_0$ and $P_1$ are given explicitly, and $\Theta= P_1^T \theta P_0$  in order to get the optimal $\theta$, one should try to obtain the optimal $\Theta$ as the multiplication map is a bijection onto the set of orthogonal matrix.
    \begin{thm}[Theorem 3.3.14, \cite{horn1994topics}] \label{orderedsingularvalue}
        Let $M$, $N \in M_{d,d}$ and denote the ordered singular values of $M$, $N$ and $MN$  by $0\leq \sigma_1(M) \leq \cdot \cdot \cdot \leq \sigma_d(M)$,  $0\leq \sigma_1(N) \leq \cdot \cdot \cdot \leq \sigma_d(N)$ and  $0\leq \sigma_1(MN) \leq \cdot \cdot \cdot \leq \sigma_d(MN)$. Then,
        \begin{eqnarray*}
            \sum_{i=1}^d \sigma_i(MN)\leq \sum_{i=1}^d \sigma_i(M)\sigma_i(N)
        \end{eqnarray*}
    \end{thm}
    From Theorem \ref{orderedsingularvalue} and the fact that all singular values of an orthogonal matrix is equal to 1,  we extract the following inequalities;

\begin{eqnarray*}
     F(\theta )&=&  \sum_{i=1}^d \left( \sigma_i(T_\Theta T_\Theta ^T) \right)^{1/2}, \\
     &\leq&  \sum_{i=1}^d  \left( \sigma_i(T_\Theta) \sigma_i ( T_\Theta ^T) \right)^{1/2}, \\
     &= & \sum_{i=1}^d  \sigma_i(A_1^{1/2}  \Theta A_0^{1/2}),\\
     &\leq& \sum_{i=1}^d  \sigma_i(A_1^{1/2}) \sigma_i ( \Theta A_0^{1/2}),\\
     &=& \langle {a_0}^\frac{1}{2}, {a_1}^\frac{1}{2} \rangle.
\end{eqnarray*}
For $\Theta = Id \longleftrightarrow \theta = P_0^T P_1$, we have:
\begin{eqnarray}
    \sup_{ \theta}\{ F(\theta)\}&=& \langle {a_0}^\frac{1}{2}, {a_1}^\frac{1}{2} \rangle, \notag \\
    &=& F(P_0^T P_1 ). \label{optimalvalueofprocrustegaussian}
\end{eqnarray}
Replacing the expression of the supremum in \ref{modifiedgaussian}, we obtain
\begin{eqnarray*}
    \bar{d}^2(\mu_0,\mu_1)
&=&  \mathrm{tr}(\Sigma_0) + \mathrm{tr}(\Sigma_1 ) -2  \sup_{ \theta}\{ F(\theta)\},\\
&=& \mathrm{tr}(P_0 A_0 {P_0}^T) + \mathrm{tr}(P_1 A_1 {P_1}^T ) -2  \langle {a_0}^\frac{1}{2}, {a_1}^\frac{1}{2}  \rangle,\\
&=& \sum_{i=1}^d a_{i,0}  + \sum_{i=1}^d a_{i,1}-2 \sum_{i=1}^d {a_{0,i}}^\frac{1}{2}{a_{1,i}}^\frac{1}{2},\\
&=& \sum_{i=1}^d \left( \sqrt{a_{i,0}} -\sqrt{a_{i,1}}\right)^2,\\
&=& \|\sqrt{a_0} - \sqrt{a_1}\|^2
\end{eqnarray*}
\end{proof}

\section{Application: Recovering Latent Gaussian Distributions Using Procrustes Wasserstein Analysis}

We consider the problem of estimating the parameters of a Gaussian distribution from observed data that have been transformed by an unknown orthogonal matrix. This problem arises in various applications, including shape analysis, manifold learning, and signal processing. Our goal is to define an unbiased estimator in the metric sense and prove its unbiasedness using the Procrustes-Wasserstein metric.
\\

We are given a set of observed vectors \( r_1, r_2, \ldots, r_n \) in a \( d \)-dimensional real space. These vectors are related to the variables \( p_1, p_2, \ldots, p_n \) through a linear transformation:
\[
r_i = V p_i \quad \text{for each } i = 1, 2, \ldots, n,
\]
where:
\begin{itemize}
    \item \( V \) is an unknown \( d \times d \) orthogonal matrix (i.e., \( V \in O(d) \)),
    \item \( p_i \) are unknown variables drawn independently and identically distributed from an unknown Gaussian distribution \( \gamma \) (i.e., \( p_i \sim \gamma \)) with 0 mean. Thus, the distribution of the observed data \( r_i \) is:
    \[
    r_i \sim V \gamma = \mathcal{N}(0, V \Sigma V^T).
    \]
\end{itemize}

We aim to estimate the parameters of the Gaussian distribution \( \gamma \) (mean and covariance matrix) from the observed vectors \( r_i \), accounting for the unknown orthogonal transformation \( V \) and observations $\{p_i, \ i\in [|1,n|]\}$.

\subsection{Set of Observables}
The \textbf{observable data} consists of \( n \) vectors in \( \mathbb{R}^d \):
\[
\mathcal{X} = \{ r_1, r_2, \dots, r_n \mid r_i \in \mathbb{R}^d \}.
\]

The \textbf{parameter space} consists of the equivalence class of Gaussian distributions under orthogonal transformations:
\[
\Theta = \{ (\gamma,~V),~~ \mid \gamma = \mathcal{N}(0, \Sigma), V \in O(d) \}.
\]

\subsection{Statistical Model}
The \textbf{statistical model} is a family of probability distributions parameterized by \( \Theta \):
\[
\mathcal{P} = \{ P_{(\gamma, V)} \mid (\gamma,V) \in \Theta \},
\]
where:
\begin{itemize}
    \item [-] \( P_{{(\gamma, V)}} \) is the distribution of the observed data \( r_i \) when the true parameter is \(\gamma\),
    \[
    {\gamma} = \mathcal{N}(0, V \Sigma V^T).
    \]
\end{itemize}

Let denote the empirical  covariance matrix given by
    \[
    \hat{\Sigma}_r = \frac{1}{n} \sum_{i=1}^n r_i r_i^T,
    \]
    with diagonal matrix \( D_{\hat{\Sigma}_r }\):
    \[
    \hat{\Sigma}_r = Q_r D_{\hat{\Sigma}_r} Q_r^T,
    \]
    where \( Q_r \in O(d) \) and \( D_{\hat{\Sigma}_r} \) is the diagonal matrix of ordered eigenvalues. \\
   Using the empirical distribution, let us construct an estimator of the true  Gaussian distribution $\gamma$ from the empirical covariance matrix $\hat{\Sigma}_r$, 
   \begin{eqnarray*}
    [\hat{\gamma}]
    &=& \{ \mathcal{N}(0, Q D_{\hat{\Sigma}_r} Q^T), \ Q\in O(d)\}.
\end{eqnarray*} 
 
\begin{mydef} 
Let \((M, d)\) be a metric space, and let \(\theta \in M\) be a random variable. An estimator \({\theta_*}\) of \(\theta\) is called Fréchet mean if:
\[
\theta_*\in \text{argmin}_{y \in M} \mathbb{E}[d({\theta}, y)^2].
\]
\end{mydef} 
\begin{prop}
    The Frechet mean of  \([ \hat{\gamma}]\)  with respect to the Procrustes-Wasserstein metric \(\bar{d}\) \eqref{staticMBB} is given by $ [\mathcal{N}(0,(\mathbb{E}[\sqrt{D_{\hat{\Sigma}_r}}])^2)]$. 
\end{prop} 
\begin{proof}
We want to show that  \[
[[\mathcal{N}(0,(\mathbb{E}[\sqrt{D_{\hat{\Sigma}_r}}])^2)]] \in \text{argmin}_{[\gamma']} \mathbb{E}[d_{PW}(\hat{\gamma}, [\gamma'])^2].
\]
The Procrustes Wasserstein distance \(\bar{d}\) between two  distributions \(\mu = \mathcal{N}(0, \Sigma_1)\) and \(\nu = \mathcal{N}(0, \Sigma_2)\) is:
\[
d_{PW}(\mu, \nu) = \inf_{\theta \in O(d)} W_2(\mu, \theta_\# \nu ).
\]
For Gaussian measure, we showed in Theorem \ref{gaussianprowass} that this simplifies to:
\[
d_{PW}(\mu, \nu) = \| \sqrt{D_{\Sigma_1}} - \sqrt{D_{\Sigma_2}} \|_2.
\]
The expected squared Procrustes Wasserstein distance is then:
\[
\mathbb{E}[d_{PW}(\hat{P}_r, [\gamma'])^2] = \mathbb{E}[\| \sqrt{D_{\hat{\Sigma}_r}} - \sqrt{D_{\Sigma'}} \|_2^2].
\]
Hence, the problem reduces to \[ \text{argmin}_{D_{\Sigma'}} \mathbb{E}[\| \sqrt{D_{\hat{\Sigma}_r}} - \sqrt{D_{\Sigma'}} \|_2^2. \]
From probability theory, $\mathbb{E}[\| \sqrt{D_{\hat{\Sigma}_r}} - \sqrt{D_{\Sigma'}} \|_2^2$ is minimizes when $ \sqrt{D_{\Sigma'}} = \mathbb{E}[\sqrt{D_{\hat{\Sigma}_r}}] $ equivalently $ D_{\Sigma'} = (\mathbb{E}[\sqrt{D_{\hat{\Sigma}_r}}])^2$.

Hence, \([\gamma'] = [\mathcal{N}(0,(\mathbb{E}[\sqrt{D_{\hat{\Sigma}_r}}])^2)]\) and thus  \[
[\mathcal{N}(0,(\mathbb{E}[\sqrt{D_{\hat{\Sigma}_r}}])^2)] = \text{argmin}_{[\gamma']} \mathbb{E}[d_{PW}(\hat{\gamma}, [\gamma'])^2].
\]
\end{proof}
\begin{rmk}
    It is worth noticing that the estimator $[\hat{\gamma}]$ is an asymptotically unbiased estimator of $[\gamma]$, this is a direct consequence of the asymptotic behavior of eigenvalues of the Wishart distribution, see Proposition 8.3 in \cite{eaton2007wishart}, the law of large number for the empirical covariance matrix and the continuity of the eigenvalue function, see Theorem 5.2 in \cite{kato2013perturbation}.
\end{rmk}
\section*{Acknowledgments}
I would like to express my sincere gratitude to   Max von Renesse for his  support throughout the course of this research.
\printbibliography
\end{document}